\documentclass[letterpaper, 10 pt, conference]{ieeeconf}  % Comment this line out if you need a4paper

\IEEEoverridecommandlockouts                              % This command is only needed if 
\usepackage[table]{xcolor}
\usepackage{amsmath} % assumes amsmath package installed

\usepackage{amssymb}  % assumes amsmath package installed
\usepackage{amsthm}
\usepackage{todonotes}
\usepackage[linesnumbered,ruled,noend]{algorithm2e}

\newtheorem{lem}{Lemma}
\newtheorem{theorem}{Theorem}
\newtheorem{problem}{Problem}

\title{\LARGE \bf
Planning with Learned Dynamics: Probabilistic Guarantees \\on Safety and Reachability via Lipschitz Constants
}

\author{Craig Knuth*$^1$, Glen Chou*$^1$, Necmiye Ozay$^1$ and Dmitry Berenson$^1$% <-this % stops a space
\thanks{*C. Knuth and G. Chou contributed equally to this work.}%
\thanks{$^1$All authors affiliated with the University of Michigan, Ann Arbor, MI 48109, US, \texttt{\{cknuth, gchou, necmiye, dmitryb\}@umich.edu}}%
\thanks{This work is funded by NSF grants IIS-1750489 and ECCS-1553873, ONR grants N00014-21-1-2118 and N00014-18-1-2501, and a National Defense Science and Engineering Graduate (NDSEG) fellowship.}%
}

\newcommand{\X}{\mathcal{X}}
\newcommand{\U}{\mathcal{U}}
\renewcommand{\S}{\mathcal{S}}
\newcommand{\B}{\mathcal{B}}
\newcommand{\Xsafe}{\X_{\textrm{safe}}}
\newcommand{\Xunsafe}{\X_{\textrm{unsafe}}}
\newcommand{\T}{\mathcal{T}}
\newcommand{\tx}{\bar{x}}
\newcommand{\tu}{\bar{u}}

\begin{document}

\maketitle
\thispagestyle{empty}
\pagestyle{empty}

%%%%%%%%%%%%%%%%%%%%%%%%%%%%%%%%%%%%%%%%%%%%%%%%%%%%%%%%%%%%%%%%%%%%%%%%%%%%%%%%
\begin{abstract}

We present a method for feedback motion planning of systems with unknown dynamics which provides probabilistic guarantees on safety, reachability, and goal stability. To find a domain in which a learned control-affine approximation of the true dynamics can be trusted, we estimate the Lipschitz constant of the difference between the true and learned dynamics, and ensure the estimate is valid with a given probability. Provided the system has at least as many controls as states, we also derive existence conditions for a one-step feedback law which can keep the real system within a small bound of a nominal trajectory planned with the learned dynamics. Our method imposes the feedback law existence as a constraint in a sampling-based planner, which returns a feedback policy around a nominal plan ensuring that, if the Lipschitz constant estimate is valid, the true system is safe during plan execution, reaches the goal, and is ultimately invariant in a small set about the goal. We demonstrate our approach by planning using learned models of a 6D quadrotor and a 7DOF Kuka arm. We show that a baseline which plans using the same learned dynamics \textit{without} considering the error bound or the existence of the feedback law can fail to stabilize around the plan and become unsafe.
\end{abstract}

%%%%%%%%%%%%%%%%%%%%%%%%%%%%%%%%%%%%%%%%%%%%%%%%%%%%%%%%%%%%%%%%%%%%%%%%%%%%%%%%
\section{INTRODUCTION}

Planning and control with guarantees on safety and reachability for systems with unknown dynamics has long been sought-after in the robotics and control community. Model-based optimal control can achieve this if the dynamics are precisely modeled, but modeling assumptions inevitably break down when applied to real physical systems due to unmodeled effects from friction, slip, flexing, etc. To account for this gap, data-driven machine learning methods and robust control seek to sidestep the need to precisely model the dynamics \textit{a priori}. While robust control can provide strong guarantees when the unmodeled component of the dynamics is small and satisfies strong structural assumptions \cite{zhou1998essentials}, such methods requires an accurate prior which may not be readily available. In contrast, machine learning methods are flexible but often lack formal guarantees, precluding their use in safety-critical applications. 
For instance, small perturbations from training data cause drastically poor and costly predictions in stock prices and power consumption \cite{mode2020adversarial}. Since even small perturbations from the training distribution can yield untrustworthy results, applying AI systems to predict dynamics can lead to unsafe, unpredictable behavior.%

To address this gap, we propose a method for planning with learned dynamics which yields probabilistic guarantees on safety, reachability, and goal invariance in execution on the true system. Our core insight is that we can determine where a learned model can be trusted for planning using the Lipschitz constant of the error (the difference between the true and learned dynamics), which also informs how well the training data covers the task-relevant domain. Under the assumption of deterministic true dynamics, we can plan trajectories in this trusted domain with strong safety guarantees for an important class of learned dynamical systems.

Specifically, with a Lipschitz constant, we can bound the difference in dynamics between a novel point (that our model was not trained on) and a training point. Since the bound grows with the distance to training points, we can naturally define a domain where the model can be trusted as the set of points within a certain distance to training points. 
Conversely, to obtain a small bound over a desired domain, it is necessary to have good training data coverage in the task-relevant domain. At a high level, to obtain a small bound on the error in a domain, we want to have good coverage over the domain and regularity of the learned model via the Lipschitz constant of the error.

Our safety and reachability guarantees ultimately rely on an overestimate of the smallest Lipschitz constant. To find an estimate that exceeds the smallest Lipschitz constant with a given probability $\rho$, we use a statistical approach based on Extreme Value Theory \cite{de2007extreme} and validate its result with a Kolmogorov-Smirnov goodness-of-fit test \cite{degroot2013probability}. If the test validates our estimate, we can choose a confidence interval \cite{degroot2013probability} with an upper bound that overestimates the true Lipschitz constant with probability $\rho$. Our method requires the estimation of three Lipschitz constants, translating to system safety and reachability guarantees which hold with a probability of at least $\rho^3$. This guarantee is fairly strong as it holds for all time, unlike many methods offering probabilistic guarantees on a per trajectory or episode basis \cite{akametalu2014reachability} \cite{berkenkamp2017safe} \cite{van2011lqg}.

If the learned dynamics have at least as many controls as states and are control-affine
(note we do not assume the true dynamics are also control-affine), then we also determine conditions for the existence of a feedback controller that tightly tracks the planned trajectory in execution under the true dynamics.
The tight tracking error bound yields favorable properties for our planner and controller: if we have a valid Lipschitz constant estimate for a sufficiently-accurate learned model, 1) we guarantee safety if no obstacle is within the tracking error of the trajectory, 2) we guarantee we can reach the goal within a small tolerance, and 3) if we can assert a feedback law that keeps the system at the goal exists, then the closed-loop system is guaranteed to remain in a small region around the goal. In this paper, we assume the learned dynamics are control-affine, deterministic, and have at least as many controls as states (such as a robotic arm under velocity control), the true dynamics are deterministic, and that independent samples of the true dynamics can be taken in the domain of interest. 
Our contributions are:

\begin{enumerate}
    \item A method to bound error between two general dynamics functions in a domain by using a Lipschitz constant
    \item A condition for uncertain control-affine systems that guarantees the existence of a feasible feedback law
    \item A planner that probabilistically guarantees safety and closed-loop stability-like properties about the goal for learned dynamics with as many controls as states
    \item Evaluation on a 7DOF Kuka arm and a 6D quadrotor
\end{enumerate}

\section{RELATED WORK}

Prior work has used data coverage and Lipschitz constant regularization to ensure properties of a learned function. \cite{dean2020robust} shows that a linearization of a nonlinear state estimator generalizes with bounded error by estimating the maximum slope (related to the Lipschitz constant) of the error. \cite{robey2020learning} learns a control barrier function (CBF) from data, ensuring its validity through Lipschitz constant regularity and by checking the CBF conditions at a finite set of points. 
In contrast, we apply and extend these ideas to plan with learned dynamics, using the Lipschitz constant of the error dynamics to provide safety, reachability, and stability-like guarantees.

More broadly, our work is related to methods for planning and control of unknown dynamics with performance guarantees. A traditional approach is robust control, which assumes a good prior on the true dynamics and that unmodeled components are tightly bounded in some set \cite{zhou1998essentials}. 
Robust control has been applied in model predictive control \cite{aswani2013provably} and Hamilton-Jacobi (HJ) reachability analysis \cite{mitchell2005time}. 
While these methods have strong guarantees, the assumption that the unmodeled dynamics can be tightly bounded %in all parts of the state and control space 
requires a accurate prior whereas our method does not, by actively keeping both planned and executed trajectories in a domain where the model can be trusted without \textit{a priori} knowledge.

Other methods use Gaussian Processes (GPs) to estimate the mean and covariance of the dynamics, providing probabilistic bounds on safety and reachability. For example, \cite{koller2018learning} probabilistically bounds the reachable set of a fixed horizon trajectory. 
Similarly \cite{akametalu2014reachability} explores the environment while ensuring (with some probability) safety via HJ reachability analysis. 
In many contexts, GPs are used to derive confidence bounds that can provide probabilistic safety guarantees \cite{berkenkamp2016safe} \cite{berkenkamp2017safe}. These methods can model dynamics with stochasticity, which our method cannot handle. However, the GP-based methods are incapable of long-horizon planning due to the 
unbounded growth of the covariance ellipse unless a known feedback controller exists. Our method does not require any prior controller, and
we can plan trajectories of arbitrary length without unbounded growth of the reachable set.

Other work performs long-horizon planning with learned models without safety or reachability guarantees. \cite{ichter2019robot} plans in learned latent spaces. 
\cite{guzzi2020path} estimates the confidence that a controller can move between states to guide planning.
\cite{mcconachie2020learning} learns when a reduced-order model can be used. %, but again offers no guarantees on safety.
Unlike \cite{ichter2019robot}-\cite{mcconachie2020learning}, our method provides safety and reachability guarantees.

Our safety guarantees rely on proving the existence of a stabilizing feedback controller in execution, like LQR-trees \cite{tedrake2009lqr}, funnel libraries \cite{MajumdarT17}, and LQG-MP \cite{van2011lqg}. Unlike these approaches, our method requires no \textit{a priori} model and can prove a feedback law exists without structural assumptions on the true dynamics (e.g. that they are polynomial).

\section{PRELIMINARIES}

Let $f: \X \times \U \rightarrow \X$ be the true unknown discrete-time dynamics where $\X$ is the state space and $\U$ is the control space, which we assume are deterministic. We define $g: \X \times \U \rightarrow \X$ to be an approximation of the true dynamics that is control-affine and therefore can be written as follows

\begin{equation}\label{eq:g}
    g(x,u) = g_0(x) + g_1(x) u. 
\end{equation}

In this paper, we represent the approximate dynamics with a neural network, though our method is agnostic to the structure of the model and how it is derived.
Let $\S = \{(x_i,u_i,f(x_i,u_i))\}_{i=1}^N$ be the training data for $g$, and let $\Psi = \{(x_j,u_j,f(x_j,u_j))\}_{j=1}^M$ be another set of samples collected near $\S$ that will be used to estimate the Lipschitz constant. We use $\bar\cdot$ to refer to data points from $\S$ or $\Psi$. We place no assumption on how $\S$ is obtained; any appropriate method (uniform sampling, perturbations from expert trajectories, etc.) may be employed%
, although we require independent and identically distributed (i.i.d.) samples for $\Psi$. A single state-control pair is written as $(x,u)$. With some abuse of notation, we write $(\tx,\tu) \in \S$ if $(\tx,\tu) = (x_i, u_i)$ for some $1 \leq i \leq N$ (similarly for $\Psi$).

A Lipschitz constant bounds how much outputs change with respect to a change in the inputs. For some function $h$, a Lipschitz constant over a domain $\mathcal{Z}$ is any number $L$ such that for all $z_1, z_2 \in \mathcal{Z}$

\begin{equation}\label{eq:lipschitz_def}
    \| h(z_1) - h(z_2) \| \leq L \| z_1 - z_2 \|
\end{equation}

Norms $\|\cdot\|$ are always the 2-norm or induced 2-norm. We define $L_{f-g}$, $L_{g_0}$, and $L_{g_1}$ as the smallest Lipschitz constants of the error $f - g$, $g_0$, and $g_1$. The input to $f-g$ is a state-control pair $(x,u)$ and its output is a state. For $g_0$, both the input and output are a state. For $g_1$, its input is a state and its output is a $\texttt{dim}(\X) \times \texttt{dim}(\U)$ matrix where \texttt{dim}$(\cdot)$ is the dimension of the space. A ball $\B_r(x)$ of radius $r$ about a point $x$ is defined as the set $\{y \enspace | \enspace \|y - x\| < r\}$, also referred to as a $r$-ball about $x$. We suppose the state space $\X$ is partitioned into safe $\Xsafe$ and unsafe $\Xunsafe$ sets (e.g., the states in collision with an obstacle).

The method consists of two major components. First, we determine a trusted domain $D \subseteq \mathcal{X} \times \mathcal{U}$ and estimate the Lipschitz constants. Second, we use $D$ to find a path to the goal satisfying our safety and reachability requirements. 

\begin{problem}\label{prob:D}
Given a learned model $g$, unknown dynamics $f$, and datasets $\Psi$ and $\S$, determine the trusted domain $D$ where $\|f(x,u)-g(x,u)\| \leq \epsilon$, for some $\epsilon > 0$. Additionally determine the Lipschitz constants $L_{f-g}$, $L_{g_0}$, and $L_{g_1}$ in $D$.
\end{problem}

\begin{problem}\label{prob:path}
Given control-affine $g$, unknown $f$, start $x_I$, goal $x_G$, goal tolerance $\lambda$, $D$, $L_{f-g}$, $L_{g_0}$, $L_{g_1}$, and $\Xunsafe$, plan a trajectory $(x_0,\ldots,x_K)$, $(u_0,\ldots,u_{K-1})$ such that $x_0 = x_I$, $x_{k+1} = g(x_k,u_k)$, $K  < \infty$, and $\|x_K - x_G\| \leq \lambda$. Additionally, under the true dynamics $f$, guarantee that closed loop execution does not enter $\Xunsafe$, converges to $\B_{\epsilon+\lambda}(x_G)$, and remains in $\B_{\epsilon+\lambda}(x_G)$ after reaching $x_K$.
\end{problem}

\section{METHOD}

Secs. \ref{sec:domain} - \ref{sec:estimating_lip} and \ref{sec:planning} - \ref{sec:alg} cover our approaches to Probs. \ref{prob:D} and \ref{prob:path}, respectively. In Sec. \ref{sec:domain}, we show how $L_{f-g}$
can establish a trusted domain and how $L_{f-g}$ can be estimated in Sec. \ref{sec:estimating_lip}. In Sec. \ref{sec:planning}, we design a planner that ensures safety, that the system remains in the trusted domain, and that a feedback law maintaining minimal tracking error exists.
We present the full algorithm in Sec. \ref{sec:alg}.

\subsection{The trusted domain}\label{sec:domain}

For many systems, we are only interested in a task-relevant domain, and it is often impossible to collect data everywhere in state space, especially for high-dimensional systems. Hence, it is natural that our learned model is only accurate near training data. With a Lipschitz constant of the error, we can precisely define how accurate the learned dynamics are in a domain constructed from the training data.
We note this derivation can also be done for systems without the control-affine assumption on the learned dynamics, and thus it can still be useful for determining where a broader class of learned models can be trusted. However, removing the control-affine structure makes controller synthesis much more difficult, and is the subject of future work.

Consider a single training point $(\tx,\tu)$ and a novel point $(x,u)$. We derive a bound on the error between the true and estimated dynamics at $(x,u)$ using the triangle inequality and Lipschitz constant of the error:
\begin{equation}
\begin{aligned}
    \|f(&x,u) - g(x,u)\| \\
    &= \|f(x,u) - g(x,u) - f(\tx,\tu) + g(\tx,\tu) \\
    & \qquad + f(\tx,\tu) - g(\tx,\tu)\| \\
    &\leq L_{f-g} \|(x,u) - (\tx,\tu)\| + \|f(\tx,\tu) - g(\tx,\tu)\|.
\end{aligned}
\end{equation}

The above relation describes the error at a novel point, but we can also generalize to any domain $D$. Define $b_T$ to be the dispersion \cite{lavalle2006planning} of $\S \cap D$ in $D$
and define $e_T$ to be the maximum training error of the learned model. Explicitly,
\begin{equation}
    b_T \doteq \underset{(x,u) \in D}{\sup} \quad \underset{(\tx,\tu) \in \S \cap D}{\min} \quad \|(x,u) - (\tx,\tu)\|
\end{equation}\begin{equation}
    e_T \doteq \underset{(\tx,\tu) \in \S \cap D}{\max} \quad \|f(\tx,\tu) - g(\tx,\tu)\|
\end{equation}
Then, we can uniformly bound the error across the entire set $D$ to yield a simple and exact relation between $f$ and $g$.
\begin{equation}\label{eq:epsilon}
    \epsilon \doteq L_{f-g} b_T + e_T
\end{equation}
\begin{equation}\label{eq:nonauto_dyn}
    \forall (x,u) \in D \quad f(x,u) = g(x,u) + \delta, \quad \|\delta\| \leq \epsilon
\end{equation}

\begin{figure}
    \centering
    \includegraphics[width=\linewidth]{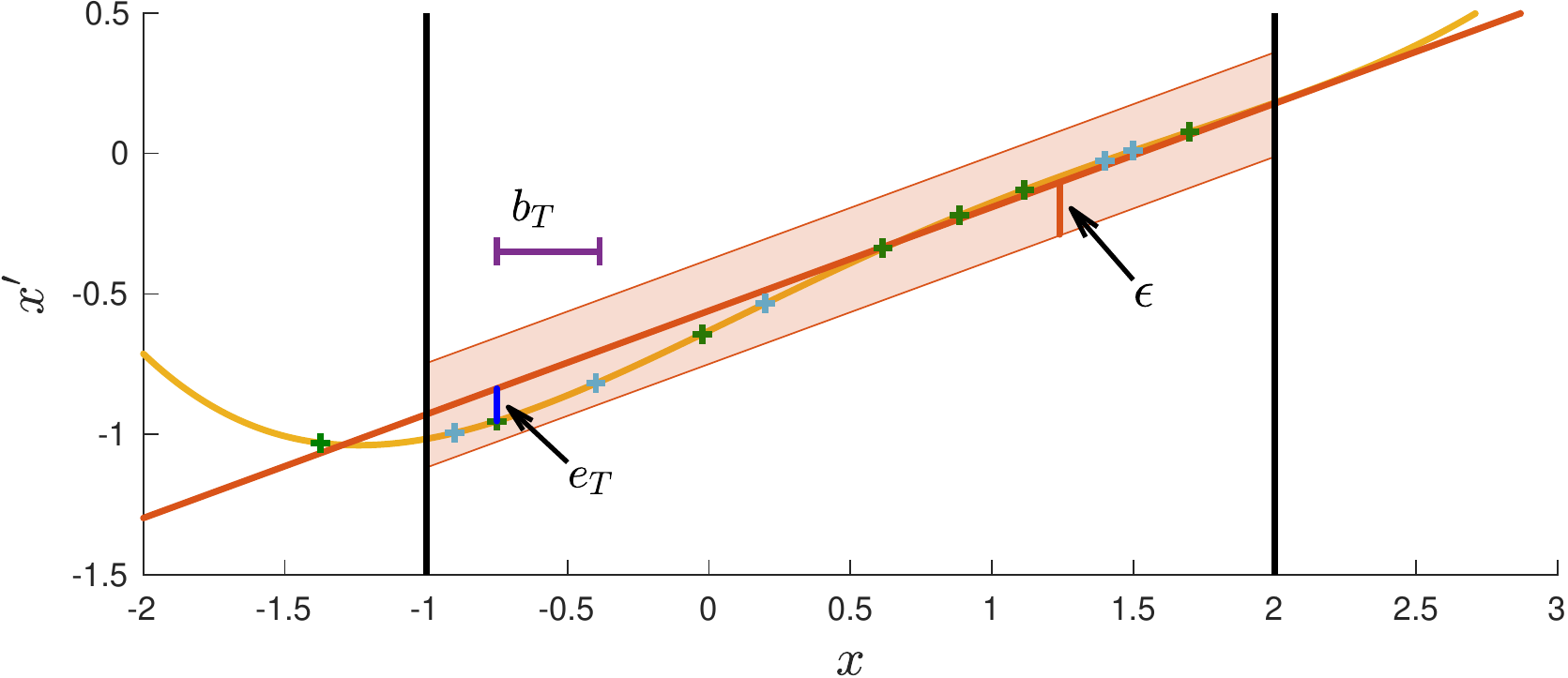}
    \caption{An example with $f(x) = x'$, $\texttt{dim}(\X) = 1$, and $\texttt{dim}(\U) = 0$. True dynamics: yellow; learned linear dynamics: orange; $\S$: green crosses; $\Psi$: light blue crosses; domain $D$: interval $[-1, 2]$, bordered in black. Here, $b_T = 0.3633$ (purple) and $e_T = 0.1161$ (blue). The Lipschitz constant of the error is $L_{f-g} = 0.1919$, yielding $\epsilon = 0.1859$. We can use this bound to ensure the difference between the learned and true dynamics is no more than $\epsilon$ in $D$ (shaded orange area). Note $L_{f-g}$ can be larger outside of $D$.}
    \label{fig:bt_et_e}
\end{figure}

See Fig. \ref{fig:bt_et_e} for an example of these quantities. For the remainder of the method, we select $D$ to be the union of $r$-balls about a subset of the training data $\S_D \subset \S$:
\begin{equation}\label{eq:def_D}
    D = \bigcup_{(\tx,\tu) \in \S_D} \, \B_{r}(\tx,\tu)
\end{equation}
In the next section we discuss selection of $\S_D$, its role in estimating $L_{f-g}$, and how $r$ is selected.

\subsection{Estimating the Lipschitz constant}\label{sec:estimating_lip}

For \eqref{eq:nonauto_dyn} to hold over all of $D$, we require that $L_{f-g}$ is a Lipschitz constant for the error. We use results from Extreme Value Theory to obtain an estimate $\hat{L}_{f-g}$ that overestimates $L_{f-g}$, i.e. $\hat{L}_{f-g} \ge L_{f-g}$, with a user-defined probability $\rho$.

We build on \cite{weibull}-\cite{weng2018evaluating}, which find an estimate $\hat{L}_h$ of the Lipschitz constant $L_h$ for a function $h(z)$ over a domain $\mathcal{Z}$ by estimating the location parameter $\gamma$ of a three-parameter reverse Weibull distribution, which for a random variable $W$ has the  cumulative distribution function (CDF) $$F_W(w) =   \begin{cases}
    \exp(-(\frac{\gamma - w}{\alpha})^\beta), & \textrm{if } w < \gamma \\
    1, & \textrm{if } w \ge \gamma.
  \end{cases}$$Here, the location parameter $\gamma$ is the upper limit on the support of the distribution, and $\alpha$ and $\beta$ are the scale and shape parameters, respectively. Consider the random variable described by the maximum slope taken over $N_L$ pairs of i.i.d. samples $\{(z_1^i,z_2^i)\}_{i=1}^{N_L}$ from $\mathcal{Z}$, i.e. $s = \max_i \frac{\Vert h(z_1^i) - h(z_2^i) \Vert}{\Vert z_1^i - z_2^i \Vert}$. From the Fisher-Tippett-Gnedenko Theorem \cite{de2007extreme}, $s$ follows one of the Frechet, reverse Weibull, or Gumbel distributions in the limit as $N_L$ approaches infinity. If $s$ follows the reverse Weibull distribution, which we validate in our results using the Kolmogorov-Smirnov (KS) goodness-of-fit test \cite{degroot2013probability} with a significance value of 0.05 (the same threshold used in \cite{weng2018evaluating}), then $L_h$ is finite and equals $\gamma$. We estimate $L_h$ using the location parameter $\hat\gamma$ of a reverse Weibull distribution fit via maximum likelihood to $N_S$ samples of $s$. Finally, we compute a confidence interval $c = \Phi^{-1}(\rho) \xi$ on $\hat\gamma$. Here $\xi$ is the standard error of the fit $\hat\gamma$, which correlates with the quality of the fit, and $\Phi(\cdot)$ is the standard normal CDF \cite{degroot2013probability}. We select the upper end of the confidence interval as our estimate $\hat{L}_h = \hat\gamma + c$, which overestimates $L_h$ with probability $\rho$. Note that increasing $\rho$ increases $c$, improving the safety probability at the cost of loosening $\hat{L}_h$, which can make planning more conservative. We also note that this probability is valid in the limit as $N_L$ approaches infinity, due to the Fisher-Tippett-Gnedenko theorem making claims only on the asymptotic distribution. We summarize the estimation method in Alg. \ref{alg:lipschitz}. 

\begin{algorithm}\label{alg:lipschitz}\small\DontPrintSemicolon
\KwIn{$N_S$, $N_L$, $\rho$}

\For{$j = 1, \ldots, N_S$}{
    sample $\{(z_1^{i,j}, z_2^{i,j})\}_{i=1}^{N_L}$ uniformly in $\mathcal{Z}$ \\
    compute $s_j = \max_i \Vert h(z_1^{i,j}) - h(z_2^{i,j}) \Vert/\Vert z_1^{i,j} - z_2^{i,j} \Vert$\\}
fit reverse Weibull to $\{s_j\}$ to obtain $\hat{\gamma}$ and standard error $\xi$ \\
validate fit using KS test with significance level 0.05 \\
\textbf{if} validated \Return $\hat{L}_h = \hat{\gamma} + \Phi^{-1}(\rho)\xi$ \textbf{else} \Return failure
\caption{Lipschitz estimation for $h(z)$ over $\mathcal{Z}$}
\end{algorithm}

We wish to choose $D$ to be large enough for planning while also keeping $L_{f-g}$ small. To achieve this, we use a filtering procedure to reduce the impact of outliers in $\S$. Let $\mu$ and $\sigma$ be the mean and standard deviation of the error over $\S$. Then, let $\S_D = \{(\tx,\tu) \in \S \enspace | \enspace \|f(\tx,\tu) - g(\tx,\tu)\| \leq \mu + a \sigma\}$ where $a$ is a user-defined parameter. Then, we run Alg. \ref{alg:select_D} in order to grow $D$. This method works by proposing values of $r$, estimating $L_{f-g}$, and increasing $r$ until $r > \epsilon$ or $L_{f-g} \geq 1$. Finding $D$ with $r > \epsilon$ and $L_{f-g} < 1$ is useful for planning (described further in Sec. \ref{sec:planning}, see \eqref{eq:select_bt}). Note that, in Euclidean spaces, $r \geq b_T$. If no filtering is done, $r = b_T$, since no point in $D$ is further than a distance $r$ from $\S_D$ and the furthest any point in $D$ can lie from a point in $\S_D$ is $r$; however, filtering shrinks $D$ and thus decreases the dispersion, making it possible that $r \ge b_T$.
The parameter $a$ should be chosen to balance the size of $D$ against the magnitude of $L_{f-g}$, which we tune heuristically.

This filtering lets us exclude regions where our learned model is less accurate, yielding smaller $e_T$. Note that filtering does not affect the i.i.d. property of the samples needed for Alg. \ref{alg:lipschitz}; it only applies a mask to the domain. We also note that Alg. \ref{alg:select_D} returns a minimum value for $r$, but a larger $r$ can be chosen as long as $L_{f-g}$ is estimated with Alg. \ref{alg:lipschitz}. A larger $r$ makes planning easier by expanding the trusted domain.

\begin{algorithm}\label{alg:select_D}\small\DontPrintSemicolon
\KwIn{$\mu$, $\sigma$, $a$, $S_D$, $\Psi$, $\alpha > 0$}

$r \leftarrow \mu + a \sigma$ \\
\While{True}{
    construct $D$ using equation \eqref{eq:def_D} \\
    estimate $L_{f-g}$ using Alg. \ref{alg:lipschitz} and $\Psi$ \\
    calculate $\epsilon$ using equation \eqref{eq:epsilon} \\
    \lIf{$L_{f-g} \geq 1$}{\Return failure} \label{line:term_cond_fail}
    \lIf{$r > \epsilon$}{\Return $r$ and $D$} \label{line:term_cond}
    \lElse{$r \leftarrow \epsilon + \alpha \quad$ \texttt{\ //\enspace $\alpha$ is a small constant}}}
\caption{Selecting $r$ and $D$}
\end{algorithm}

While we never explicitly address the assumption that the true dynamics are deterministic, the estimated Lipschitz constant may be unbounded in the stochastic case, such as when two samples have the same inputs but different outputs due to noise, causing a division by 0 in line 3 of Alg. \ref{alg:lipschitz}.

$L_{g_0}$ and $L_{g_1}$ may also be estimated with Alg. \ref{alg:lipschitz}, which we employ in the results. Alternatively, \cite{fazlyab2019efficient} can give tight upper bounds on the Lipschitz constant of neural networks, though it could not scale to the networks used in our results. Other approaches \cite{JordanD20} improve scalability at the cost of looser Lipschitz upper bounds, and will be examined in the future. 

\subsection{Planning}\label{sec:planning}

We want to plan a trajectory from start $x_I$ to goal $x_G$ using the learned dynamics while remaining in $\Xsafe$ in execution. We constrain the system to stay inside $D$, as model accuracy may degrade outside of the trusted domain.
We develop a planner similar to a kinodynamic RRT \cite{lavalle2001randomized}, growing a search tree $\T$ by sampling controls that steer towards novel states until we reach the goal. If a path is found the we can ensure the goal is reachable with safety guarantees.

\subsubsection{Staying inside $D$}\label{sec:in_D}

To remain inside the set $D$, we introduce another set $D_\epsilon := D \ominus \B_\epsilon(0)$, which is the Minkowski difference between $D$ and a ball of radius $\epsilon$. 
Every point in $D_\epsilon$ is at least a distance of $\epsilon$ from any point in the complement of $D$. Since the learned dynamics differs from the true dynamics by at most $\epsilon$ in $D$, controlling to a point in $D_\epsilon$ under the learned dynamics ensures the system remains within $D$ under the true dynamics (see Fig. \ref{fig:d_eps}).

\begin{figure}
    \centering
    \includegraphics[width=0.4\textwidth]{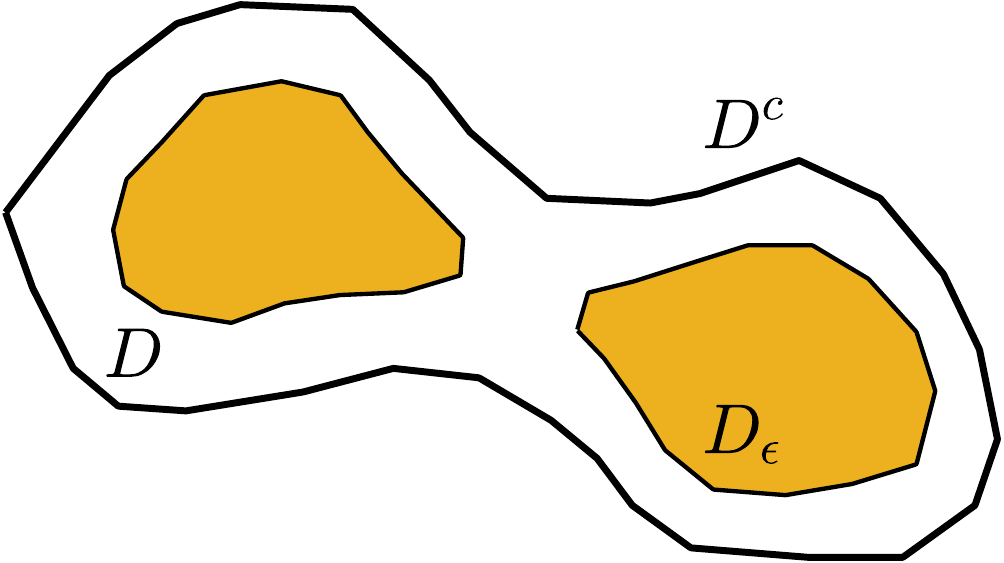}
    \caption{Visualizing $D$ (boundary in black), $D_\epsilon$ (yellow), and $D^c$ (complement of $D$). Each point in $D_\epsilon$ is at least $\epsilon$ distance away from $D^c$. If the system is controlled to a point in $D_\epsilon$ from anywhere in $D$ under the learned dynamics, then it remains in $D$ under the true dynamics.} \label{fig:d_eps}
\end{figure}

How do we determine if a query point $(x,u)$ is inside of $D_\epsilon$? Since we define $D$ to be a union of balls \eqref{eq:def_D}, it would suffice to find a subset of training points $\mathcal{W} \subset S_D$ such that the union of $r$-balls about the training points completely covers an $\epsilon$-ball about $(x,u)$. Explicitly,
\begin{equation}\label{eq:coverage}
    \bigcup_{(\tx,\tu) \in \mathcal{W}} \, \B_{r}(\tx,\tu) \supset \B_\epsilon(x,u)
\end{equation}

In general, checking \eqref{eq:coverage} is difficult, but if $L_{f-g} < 1$ and
\begin{equation}\label{eq:select_bt}
    r > \frac{e_T}{1 - L_{f-g}},
\end{equation}

\noindent then only one training point within a distance $r - \epsilon$ is needed to ensure a query point is in $D_\epsilon$ (see Fig. \ref{fig:small_l}). Note by Alg. \ref{alg:select_D} lines \ref{line:term_cond_fail}-\ref{line:term_cond}, either \eqref{eq:select_bt} is guaranteed or $L_{f-g} \geq 1$, in which we return failure.

\begin{lem}
If $L_{f-g} < 1$ and $r$ is selected according to equation \eqref{eq:select_bt}, then a point $(x,u)$ is in $D_\epsilon$ if there exists $(\tx,\tu) \in S$ such that $\|(x,u) - (\tx,\tu)\| \leq r - \epsilon$. 
\end{lem}

\begin{proof}To prove, note we can rearrange terms in equation \eqref{eq:select_bt} to get $r > L_{f-g} r + e_T \ge \epsilon$. If there exists $(\tx,\tu) \in S_D$ such that $\|(x,u) - (\tx,\tu)\| \leq r - \epsilon$, then $\B_\epsilon(x,u) \subset \B_{r}(\tx,\tu) \subset D$ since no point in $\B_\epsilon(x,u)$ is further than $r$ distance from $(\tx,\tu)$. Since $\B_\epsilon(x,u) \subset D$, $(x,u)$ is at least $\epsilon$ distance from any point in $D^c$ and therefore $(x,u) \in D_\epsilon$.
\end{proof}

In order to ensure $L_{f-g} < 1$, since it is derived from the training data and learned model, we must train a learned model that is sufficiently accurate (i.e. low error on $\S \cup \Psi$). In our experiments, it was enough to minimize mean squared error over the training set to learn models with this property.

\begin{figure}
    \centering
    \includegraphics[width=0.8\linewidth]{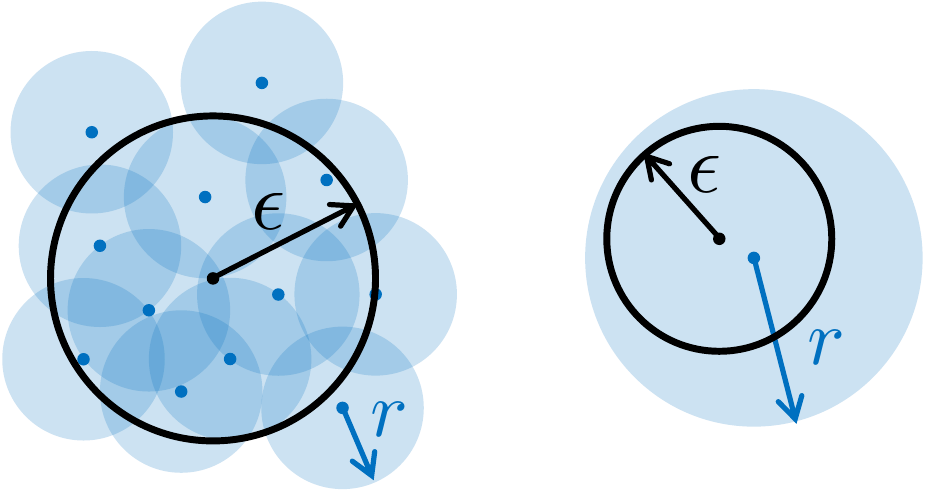}
    \caption{Illustrating the advantage of $L_{f-g} < 1$ and $r$ selected according to \eqref{eq:select_bt}. An $\epsilon$-ball about a query point is shown in black; $r$-balls about training data are shown in blue. \textbf{Left}: $L_{f-g} > 1$, therefore requiring many training points to cover an $\epsilon$-ball about the query point. \textbf{Right}: $L_{f-g} < 1$ and $r$ is selected according to \eqref{eq:select_bt}. Under these conditions, only one training point within a $r - \epsilon$ distance ensures an $\epsilon$-ball about the $(x,u)$ is entirely in $D$, ensuring that the query point is in $D_\epsilon$.} \label{fig:small_l}
\end{figure}

To ensure that the resulting trajectory remains in $D_\epsilon$, we ensure that corresponding pairs of state and control lie in $D_\epsilon$ at each step. In growing the search tree $\T$, we break down this requirement into two separate checks, the first of which optimistically adds states to the search tree and the second that requires pairs of states and controls to lie in $D_\epsilon$. To illustrate, suppose we sample a new configuration $x_\textrm{new}$ and grow the tree from some $x$ to $x_\textrm{new}$. At this point, when sampling a control $u$ to steer from $x$ to $x_\textrm{new}$ we enforce that $(x,u) \in D_\epsilon$ (see line \ref{line:dist_check} in Alg. \ref{alg:rrt}). However, how do we know the resulting state, $x' = g(x,u)$, will lie in $D_\epsilon$? Since $x'$ is a state and not a state-control pair, the above question is not well defined. Instead, we perform an optimistic check in adding $x'$ which requires that there exists some $\hat{u}$ such that $(x',\hat{u}) \in D_\epsilon$ (see line \ref{line:opt_check2} in Alg. \ref{alg:rrt}). In turn, when growing the search tree from $x'$ to some other sampled point $x'_\textrm{new}$ we ensure that the pair of state and newly sampled control $u'$ lies in $D_\epsilon$, i.e. $(x',u') \in D_\epsilon$.

\subsubsection{One step feedback law}\label{sec:feedback}

To prevent drift in execution, we also seek to ensure the trajectory planned with RRT can be tracked with minimal error. One key requirement to guarantee a feedback law exists is that the system is sufficiently actuated under the learned dynamics. This requires that $\texttt{dim}(\U) \geq \texttt{dim}(\X)$. The check for sufficient actuation is done on a per state basis and can be done as we grow $\T$. This feedback law ensures that, under the learned dynamics, we can return to a planned trajectory in exactly one step.

\begin{figure}
    \centering
    \includegraphics[width=0.42\textwidth]{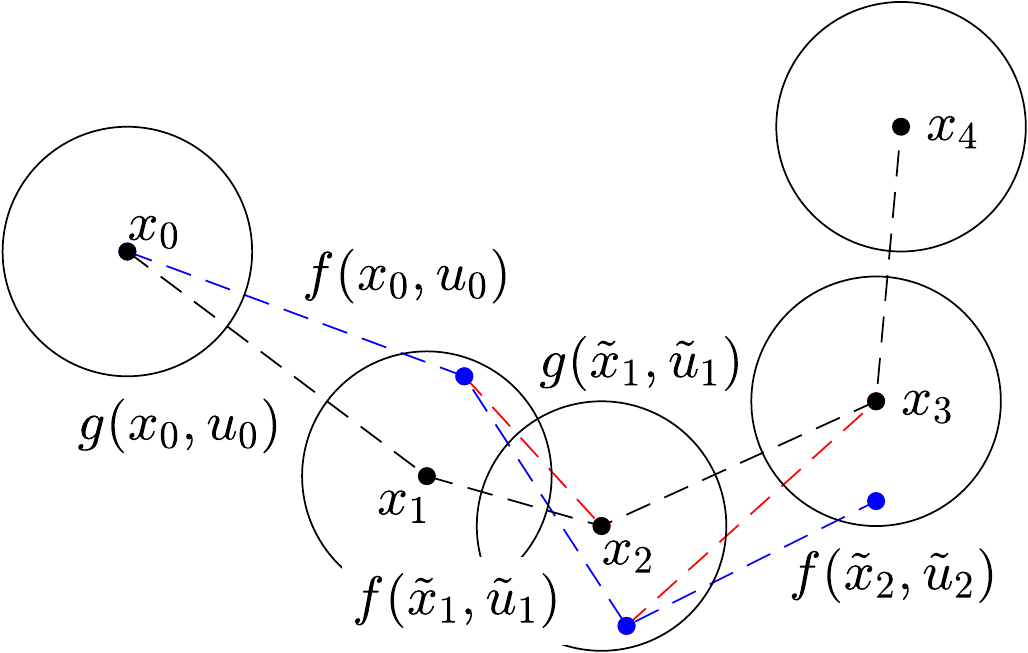}
    \caption{The one-step feedback law: plan with the learned dynamics (dashed black); rollout with the true dynamics (blue); prediction with the learned dynamics using the feedback law (red). 
    At each point, we use \eqref{eq:pert_lls} to find a feedback control $\tilde{u}_k$ so $x_{k+1} = g(\tilde{x}_k, \tilde{u}_k)$. We arrive within $\epsilon$ of the next state under the true dynamics. This repeats until we reach the goal.} \label{fig:one_step}
\end{figure}

Suppose we are executing a trajectory $(x_0, \ldots, x_K)$ with corresponding control $(u_0, \ldots, u_{K-1})$ planned with the learned dynamics, and the system is currently at $x_{k-1}$. Under the learned dynamics, the plan is to move to $x_{k} = g(x_{k-1},u_{k-1})$, but, under the true dynamics, the system will end up at some $\tilde{x}_{k} = f(x_{k-1},u_{k-1})$ which is no more than an $\epsilon$ distance from $x_{k}$. Our goal is to find an input $\tilde{u}_{k}$ such that $x_{k+1} = g(\tilde{x}_k, \tilde{u}_k)$. If this one-step feedback law exists for all $1 \leq k \leq K-1$, it ensures the executed trajectory stays within $\epsilon$ distance of the planned trajectory (see Fig. \ref{fig:one_step}).

With Lipschitz constants $L_{g_0}$ and $L_{g_1}$, we can bound how much the learned dynamics varies in the $\epsilon$-ball about $x_k$.
\begin{equation}\label{eq:g_bound}\small
    \forall \tilde{x}_k \in \mathcal{B}_\epsilon(x_k) \enspace g(\tilde{x}_k,u) = g_0(x_k) + \Delta_0 + (g_1(x_k) + \Delta_1)u
\end{equation}

\noindent where $\|\Delta_0\| \leq L_{g_0}\epsilon$ and  $\|\Delta_1\| \leq L_{g_1}\epsilon$. With \eqref{eq:g_bound}, the existence of $\tilde{u}_k$ is informed by a perturbed linear equation:
\begin{equation}\label{eq:pert_lls}
    \begin{aligned}
        x_{k+1} &= g(\tilde{x}_k,\tilde{u}_k), \quad \tilde{x}_k \in  \mathcal{B}_\epsilon(x_k) \\
        & \hspace{-25pt} \Rightarrow x_{k+1} = g_0(x_k) + \Delta_0 + (g_1(x_k) + \Delta_1)\tilde{u}_k \\
        & \hspace{-25pt} \Rightarrow A\tilde{u}_k = b \\
        \text{with} \enspace A &= g_1(x_k) + \Delta_1 \enspace \text{and} \enspace b = x_{k+1} - g_0(x_k) - \Delta_0
    \end{aligned}\hspace{-15pt}
\end{equation}

Prior to execution, we seek to answer two questions: when does $\tilde{u}_k$ exist and does $\tilde{u}_k$ lie in the control space $\U$ (for instance in the presence of box constraints)? Results from the literature \cite{lotstedt1983perturbation} give a bound on the difference between the nominal solution $u_k$ and perturbed solution $\tilde{u}_k$, 
\begin{equation}
    \Vert u_k - \tilde{u}_k \Vert \le \frac{\Vert g_1(x_k)^+\Vert(\Vert \Delta_1 \Vert \Vert u_k \Vert + \Vert \Delta_0 \Vert)}{1 - \Vert g_1(x_k)^+ \Vert \Vert \Delta_1 \Vert} \doteq u_\textrm{pert},
\end{equation}
where $g_1(x_k)^+$ is the pseudo-inverse of $g_1(x_k)$ (in general $g_1(x_k)$ is not square). We can use this bound to ensure that $\tilde{u}_k$ is guaranteed to lie in $\U$ by enforcing that $u_k + u_\textrm{pert} \mathbf{1}_\infty \subseteq \mathcal{U}$, where $\mathbf{1}_\infty$ is the unit infinity-norm ball. 
Furthermore, $A$ may become singular if $1 - \|g_1(x_k)^+\| \, \|\Delta_1\| \leq 0$. In this case, $\tilde{u}_k$ is not guaranteed to exist.

If $\tilde{u}_k$ exists and satisfies the control constraints for all $1 \leq k \leq K-1$, then we ensure that the system will track the path up to an $\epsilon$ error under the one-step feedback law. 
In planning, we add the existence of a valid one step feedback law as a check when growing the search tree. Formally:

\begin{theorem}\label{th:tracking}
For trajectory $(x_0, \ldots, x_K)$ and $(u_0, \ldots u_{K-1})$, if the solution to the perturbed linear equation \eqref{eq:pert_lls}, $\tilde{u}_k$, exists for all $k \in \{1, \ldots, K-1\}$, then under the true dynamics $\|\tilde{x}_k - x_k\| \leq \epsilon$ for all $k$, given $L_{f-g}$, $L_{g_0}$, and $L_{g_1}$ are each an overestimate of the true Lipschitz constant of $f-g$, $g_0$, and $g_1$, respectively.
\end{theorem}

\begin{proof}Proof by induction. For the induction step, assume $\|\tilde{x}_k - x_k\| \leq \epsilon$ for some $k$. Since $\tilde{x}_k \in \B_\epsilon(x_k)$, the perturbed linear equation \eqref{eq:pert_lls} is valid. If a solution exists, then $x_{k+1} = g(\tilde{x}_k, \tilde{u}_k)$ and $\|f(\tilde{x}_k,\tilde{u}_k) - x_{k+1}\| \leq \epsilon$. This satisfies the induction step. For the base case, we have $g(x_0,u_0) = x_1$ and $\|f(x_0,u_0) - x_1\| \leq \epsilon$. Thus, for all $k$, $\|\tilde{x}_k - x_k\| \leq \epsilon$.%\qed
\end{proof}

\subsubsection{Ensuring safety and invariance about the goal}\label{sec:safety_stability}

Since it is guaranteed by Thm. \ref{th:tracking} that $\|\tilde{x} - x_k\| \leq \epsilon$, we check that $\B_\epsilon(x_k) \subset \Xsafe$ for each $x_k$ on the path to ensure safety.

The exact nature of this check depends on the system and definition of $\Xunsafe$. For example, in our experiments on quadrotor, the state includes the quadrotor's position in $\mathbb{R}^3$ and $\Xunsafe$ is defined by unions of boxes in $\mathbb{R}^3$. By defining a bounding sphere that completely contains the quadrotor, we can verify a path is safe via sphere-box intersection. With the Kuka arm, we randomly sample joint configurations in an $\epsilon$-ball about states, transform the joint configurations via forward kinematics, and check collisions in workspace.
While this method is not guaranteed to validate the entire ball around a state, in practice no collisions resulted from execution of plans. Another approach computes a free-space bubble \cite{quinlan1994real} around a given state $x$ and check if it contains $\B_\epsilon(x)$, however this is known to be conservative.

To stay near the goal after executing the trajectory, we use the same perturbed linear equation to ensure the existence of a one-step feedback law. Here, rather than checking the next state along the trajectory is reachable from the previous, we check that the final state is reachable from itself, i.e. $x_K$ is reachable from $x_K$. Similar to the arguments above, we can repeatedly execute the feedback law to ensure the system remains in an $(\epsilon + \lambda)$-ball about the goal. Formally, we have:

\begin{theorem}
If the solution, denoted $u_\textrm{st}$, to the perturbed linear equation exists for $A = g_1(x_K) + \Delta_1$ and $b = x_K - g_0(x_K) - \Delta_0$ for all $x \in \B_\epsilon(x_K)$, then the closed loop system will remain in $\B_{\epsilon+\lambda}(x_G)$, given $L_{f-g}$, $L_{g_0}$, and $L_{g_1}$ are each an overestimate of the true Lipschitz constant of $f-g$, $g_0$, and $g_1$, respectively.
\end{theorem}

\begin{proof}
By Thm. \ref{th:tracking}, $\|\tilde{x}_K - x_K\| \leq \epsilon$. Thus, if the solution to the perturbed linear equation with $A = g_1(x_K) + \Delta_1$ and $b = x_K - g_0(x_K) - \Delta_0$ exists and is valid then $g(\tilde{x}_K,u_\textrm{st}) = x_K$ and $\|f(\tilde{x}_K,u_\textrm{st}) - x_K\| \leq \epsilon$. Since $\|x_K - x_G\| \leq \lambda$, the system remains in $\B_{\epsilon + \lambda}(x_K)$ by the triangle inequality.%\qed
\end{proof}

To close, we note that the overall safety and invariance probability of our method is $\rho^3$, arising from our need to estimate three Lipschitz constants: $L_{f-g}$, $L_{g_0}$, and $L_{g_1}$. Given independent samples for overestimating each constant with probability $\rho$ via Alg. \ref{alg:lipschitz}, the overall correctness probability is the product of the correctness of each constant, i.e. $\rho^3$.

\subsection{Algorithm}\label{sec:alg}

We present our full method, \textbf{L}earned \textbf{M}odels in \textbf{T}rusted \textbf{D}omains (LMTD-RRT), in Alg. \ref{alg:rrt}. In practice, we implemented \texttt{SampleState} and \texttt{SampleControl} in two different ways: uniform sampling and perturbations from training data. Sampling perturbations (up to a norm of $r - \epsilon$) does not exclude valid $(x,u)$ pairs since all points in $D_\epsilon$ lie within $r - \epsilon$ from a training point, and, in cases where $D_\epsilon$ is a relatively small volume, can yield a faster search. However, it also biases samples near regions where training data is more dense. We define the set $\S_\X = \{\tx \enspace | \enspace \exists \tu \enspace \text{s.t.} \enspace (\tx,\tu) \in \S_D\}$ to describe the optimistic check described in Sec. \ref{sec:in_D}. 
\texttt{NN} finds the nearest neighbor and \texttt{OneStep} checks that a valid feedback exists as described in Sec. \ref{sec:feedback}. \texttt{Model} evaluates the learned dynamics and \texttt{InCollision} checks if an $\epsilon$-ball is in $\Xsafe$ as described in Sec. \ref{sec:safety_stability}.

\begin{algorithm}\label{alg:rrt}\small
\KwIn{$x_I$, $x_G$, $S_\X$, $S_D$, $r$, $\epsilon$, $\lambda$, $N_\textrm{samples}$, goal\_bias}
\SetKwFunction{SampleState}{SampleState}
\SetKwFunction{SampleControl}{SampleControl}
\SetKwFunction{NearestNeighbor}{NN}
\SetKwFunction{SteerInDEpsilon}{SteerInDEpsilon}
\SetKwFunction{OneStepReachable}{OneStep}
\SetKwFunction{Model}{Model}
\SetKwFunction{ConstructPath}{ConstructPath}
\SetKwFunction{InCollision}{InCollision}

$\T \leftarrow \{x_I\}$ \\

\While{\upshape True}{
    \While{\upshape $\neg$sampled} {
        $x_{\textrm{new}} \leftarrow $ \SampleState{\upshape goal\_bias} \\
        \If{\upshape $\|x_{\textrm{new}} - $ \NearestNeighbor{\upshape $S_\X$, $x_{\textrm{new}}$}$\| \leq r - \epsilon$}{\label{line:opt_check}
            sampled $\leftarrow$ True
        }
    }
    $x_{\textrm{near}} \leftarrow $ \NearestNeighbor{\upshape $\T$, $x_{\textrm{new}}$} \\
    $i \leftarrow 0, u_{\textrm{best}} \leftarrow \emptyset, x_{\textrm{best}} \leftarrow \emptyset, d \leftarrow \infty$ \\
    \While{\upshape $i < N_\textrm{samples}$}{
        $u \leftarrow$ \SampleControl{} \\
        \If{$\| (x_{\textrm{near}}, u) - $\NearestNeighbor{\upshape $S_D$, $(x_{\textrm{near}}, u)$}$\| \leq r - \epsilon$}{\label{line:dist_check}
            $x_{\textrm{next}} \leftarrow$ \Model{\upshape $x_{\textrm{near}}$, $u$} \\
            \If{\upshape \OneStepReachable{$x_{\textrm{near}}$, $u$, $x_{\textrm{next}}$} $\land$\label{line:one_step} \\
            $\quad \|x_{\textrm{next}} - $ \NearestNeighbor{\upshape $S_\X$, $x_{\textrm{next}}$}$\| \leq r - \epsilon$ $\land$ \label{line:opt_check2} \\
            $\quad \|x_{\textrm{next}} - x_G\| < d$ $\land$ \\
            $\quad \neg$\InCollision{$x_{\textrm{next}}, \epsilon$}}{
                $u_{\textrm{best}} \leftarrow u$, $x_{\textrm{best}} \leftarrow x_{\textrm{next}}$ \\
                $d \leftarrow \|x_{\textrm{next}} - x_G\|$ \\
            }
        }
        $i \leftarrow i + 1$
    }
    \If{\upshape $u_{\textrm{best}}$}{
        $\T \leftarrow \T \cup \{x_{\textrm{best}}\}$ \\
        \If{\upshape $\|x_{\textrm{best}} - x_G \| \leq $ $\lambda$}{
            return \ConstructPath{\upshape $\T$,  $x_{\textrm{best}}$}
        }
    }
}

\caption{LMTD-RRT}
\end{algorithm}

Once a plan has been computed, it can be executed in closed-loop with Alg. \ref{alg:rollout}. \texttt{ModelG0} and \texttt{ModelG1} evaluate $g_0$ and $g_1$ of the learned model. \texttt{SolveLE} solves the linear equation and \texttt{Dynamics} executes the true dynamics $f$.

\begin{algorithm}\label{alg:rollout}\small
\KwIn{$\{x_{k}\}_{k=0}^K$, $\{u_{k}\}_{k=0}^{K-1}$}
\SetKwFunction{SolveLLS}{SolveLE}
\SetKwFunction{Dynamics}{Dynamics}
\SetKwFunction{ModelF}{ModelG0}
\SetKwFunction{ModelG}{ModelG1}

$\tilde{x}_0 \leftarrow x_0$, $k \leftarrow 0$ \\
\For{$k = 1\ldots n-1$} {
    $b \leftarrow x_{k+1} - $ \ModelF{\upshape $\tilde{x}_{k}$}, $A \leftarrow $\ModelG{\upshape $\tilde{x}_{k}$} \\
    $\tilde{u}_{k} \leftarrow $\SolveLLS{$A$, $b$} \\
    $\tilde{x}_{k+1} \leftarrow$ \Dynamics{\upshape $\tilde{x}_{k}$, $\tilde{u}_{k}$} \\
}

\caption{LMTD-Execute}
\end{algorithm}

%%%%%%%%%%%%%%%%%%%%%%%%%%%%%%%%%%%%%%%%%%%%%%%%%%%%%%%%%%%%%%%%%%%%%%%%%%%%%%%%

%%%%%%%%%%%%%%%%%%%%%%%%%%%%%%%%%%%%%%%%%%%%%%%%%%%%%%%%%%%%%%%%%%%%%%%%%%%%%%%%

%%%%%%%%%%%%%%%%%%%%%%%%%%%%%%%%%%%%%%%%%%%%%%%%%%%%%%%%%%%%%%%%%%%%%%%%%%%%%%%%
\section{RESULTS}

We present results on 1) a 2D system to illustrate the need for remaining near the trusted domain, 2) a 6D quadrotor to show scaling to higher-dimensional systems, and 3) a 7DOF Kuka arm simulated in Mujoco \cite{todorov2012mujoco} to show scaling to complex dynamics that are not available in closed form. Using $\rho = 0.975$, we plan with LMTD-RRT and rollout the plans in open-loop (no computation of $\tilde{u}_k$) and closed-loop (Alg. \ref{alg:rollout}). We compare with a na\"ive kinodynamic RRT that skips the checks on lines \ref{line:opt_check}, \ref{line:dist_check}, \ref{line:one_step}-\ref{line:opt_check2} of Alg. \ref{alg:rrt} in both open and closed loop. 
See the video for experiment visualizations.

\subsection{2D Sinusoidal Model}

To aid in visualization, we demonstrate LMTD-RRT on a 2D system with dynamics $f(x,u) = f_0(x) + f_1(x)u$:
\begin{equation*}\small
    f_0(x) = \begin{bmatrix} x \\ y \end{bmatrix} + \Delta T \begin{bmatrix} 3 \sin(0.3(x + 4.5)\big) \big\vert \sin\big(0.3(y + 4.5)) \big\vert \\ 
    3 \sin(0.3(y + 4.5)\big) \big\vert \sin\big(0.3(x + 4.5)) \big\vert \end{bmatrix}
\end{equation*}
\begin{equation*}\small
    f_1(x) = \Delta T \begin{bmatrix} 1 + 0.05\cos(y) & 0 \\ 
    0 & 1 + 0.05\sin(x) \end{bmatrix}
\end{equation*}

\noindent where $\Delta T = 0.2$. We are given 9000 training points $(x_i, u_i, f(x_i,u_i))$, where $x_i$ is drawn uniformly from an `L'-shaped subset of $\mathcal{X}$ (see Fig. \ref{fig:sinusoid}) and $u_i$ is drawn uniformly from $\mathcal{U} = [-1, 1]^2$. $g_0(x)$ and $g_1(x)$ are modeled with separate neural networks with one hidden layer of size 128 and 512, respectively. We select $a=3$ in Alg. \ref{alg:select_D}. 1000 more samples are used to estimate $L_{f-g}$ via Alg. \ref{alg:lipschitz}, which we validate with a KS test with a $p$ value of $0.56$, far above the $0.05$ threshold significance value. We obtain $\hat\gamma = 0.117$ and $c = 6.85\times 10^{-4}$, giving $\epsilon = 0.215$ over $D$.%\todo{D: should report what you got from $\gamma$ and $c$}

See Fig. \ref{fig:sinusoid} for examples of the nominal, open-loop, and closed-loop trajectories planned with LMTD-RRT and a na\"ive kinodynamic RRT. The plan computed with LMTD-RRT remains in regions where we can trust the learned model (i.e. within $D_\epsilon$) and the closed-loop execution of the trajectory converges to $\B_{\epsilon+\lambda}(x_G)$. In contrast, both the open-loop and closed-loop execution of the na\"ive RRT plan diverge.
We provide statistics in Table \ref{table:stats_sinusoid} of maximum $\ell_2$ tracking error $\max_{i \in \{1, \ldots, T\}}\Vert \tilde{x}_i - x_i\Vert$ and final $\ell_2$ distance to the goal $\Vert \tilde{x}_T - x_G\Vert_2$ for both the open loop (OL) and closed loop (CL) variants, averaged over 70 random start/goal states. To give the baseline an advantage, we fix the start/goal states and plan with na\"ive RRT using two different dynamics models: 1) the same learned dynamics model used in LMTD-RRT and 2) a learned dynamics model with the same hyperparameters trained on the full dataset ($10^4$ datapoints), and report the statistics on the minimum of the two errors. The worst case tracking error for the plan computed with LMTD-RRT was $0.199$, which is within the guaranteed tracking error bound of $\epsilon = 0.215$, while despite the data advantage, plans computed with na\"ive RRT suffer from higher tracking error. Average planning times for LMTD-RRT and na\"ive RRT are 4.5 and 17 seconds, respectively.
Overall, this suggests that planning with LMTD-RRT avoids regions where model error may lead to poor tracking, unlike planning with a na\"ive RRT.

\begin{figure}
    \centering
    \includegraphics[width=\linewidth]{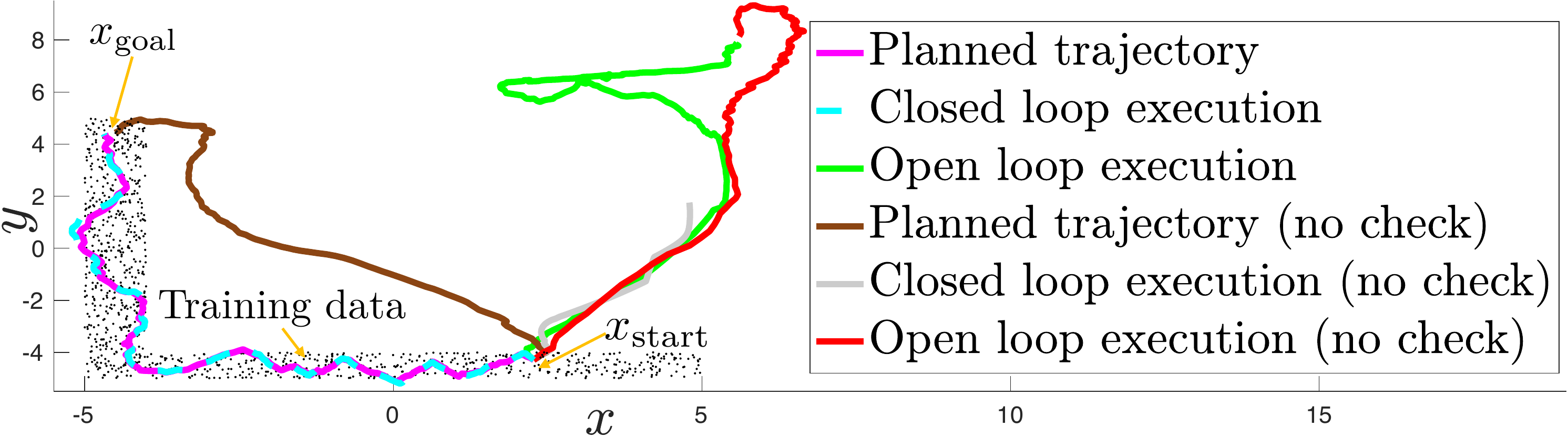}
    \caption{2D sinusoidal dynamics. The LMTD-RRT plan (magenta) stays in $D$ and ensures a valid feedback law exists at each step. The plan can be tracked within $\epsilon$ under closed loop control (cyan). If feedback is not applied, the system drifts to the edge of the trusted domain, exits, and diverges (green). The na\"ive RRT plan (brown) does not consider $D$, and does not reach the goal under closed loop (grey) or open loop (red) control.
    }
    \label{fig:sinusoid}
\end{figure}

\begin{table}\centering
\begin{tabular}{ c | c | c }
  & LMTD-RRT & Na\"ive kino. RRT \\\hline
 \cellcolor{lightgray!50!} \hspace{-7pt}Max. trck. err. (CL)\hspace{-5pt} & \cellcolor{lightgray!50!} 0.099 $\pm$ 0.036 (0.199)& \cellcolor{lightgray!50!} 8.746 $\pm$ 4.195 (15.21) \\ 
 \hspace{-7pt}Goal error (CL)\hspace{-5pt} & 0.039 $\pm$ 0.020 (0.113)& 7.855 $\pm$ 3.851 (14.78) \\ 
 \cellcolor{lightgray!50!} \hspace{-7pt}Max. trck. err. (OL)\hspace{-5pt} & \cellcolor{lightgray!50!} 12.84 $\pm$ 4.444 (20.92)& \cellcolor{lightgray!50!} 10.39 $\pm$ 1.962 (15.31)\\  
 \hspace{-7pt}Goal error (OL)\hspace{-5pt} & 12.40 $\pm$ 4.576 (20.92)& 10.12 $\pm$ 1.762 (15.20)\\  
\end{tabular}
\caption{Sinusoid errors in closed loop (CL) and open loop (OL). \\ Mean $\pm$ standard deviation (worst case).}
\label{table:stats_sinusoid}
\end{table}

\subsection{6D Quadrotor Model}

\begin{figure}
    \centering
    \includegraphics[width=\linewidth]{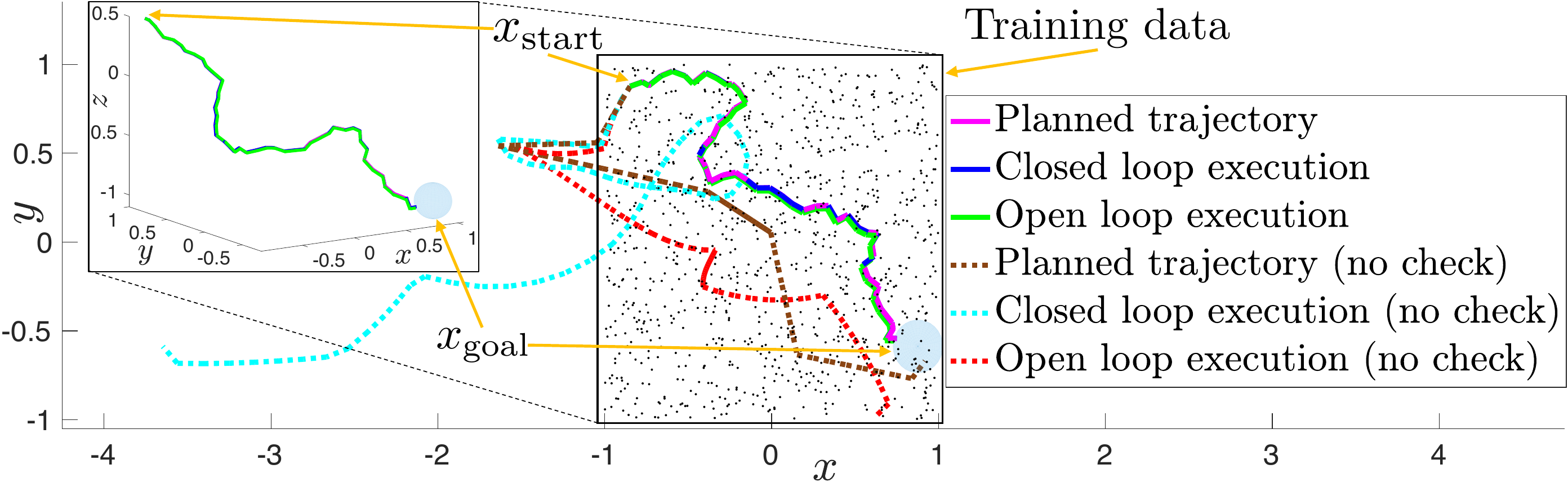}
    \caption{Quadrotor tracking example. The trajectory planned with LMTD-RRT (magenta) is tracked in closed loop (blue) and reaches the goal. The open loop (green) also converges near the goal, but not as close as the closed loop. The na\"ive RRT produces a plan (brown) that leaves the trusted domain. Thus, both the open (red) and closed (light blue) loop rapidly diverge.}
    \label{fig:quadrotor}
\end{figure}

We evaluate our method on 6-dimensional fully-actuated quadrotor dynamics \cite{quadrotor} with state $x = [\chi, y, z, \phi, \theta, \psi]^\top$, where $f(x,u) = f_0(x) + f_1(x)u$, $f_0(x) = x$ and $f_1(x) = $
\begin{equation*}\footnotesize
\Delta T \begin{bmatrix} c_\theta c_\psi &\hspace{-6pt} -c_\phi s_\psi + c_\psi s_\phi s_\theta & \hspace{-7pt}s_\psi s_\phi + c_\phi c_\psi s_\theta & \hspace{-5pt}0 & \hspace{-3pt}0 & \hspace{-3pt}0 \\
c_\theta s_\psi &\hspace{-6pt} c_\phi c_\psi + s_\phi s_\psi s_\theta & \hspace{-7pt}-c_\psi s_\phi + c_\phi s_\psi s_\theta & \hspace{-5pt}0 & \hspace{-3pt}0 & \hspace{-3pt}0 \\
-s_\theta &\hspace{-6pt} c_\theta s_\phi & \hspace{-7pt}c_\phi c_\theta & \hspace{-5pt}0 & \hspace{-3pt}0 & \hspace{-3pt}0 \\
0 &\hspace{-5pt} 0 & 0 & \hspace{-5pt}1 & \hspace{-3pt}s_\phi t_\theta & \hspace{-3pt}c_\phi t_\theta \\
0 &\hspace{-5pt} 0 & 0 & \hspace{-5pt}0 & \hspace{-3pt}c_\phi     &    \hspace{-3pt}   -s_\phi \\
0 &\hspace{-5pt} 0 & 0 & \hspace{-5pt}0 & \hspace{-3pt}s_\phi c_\theta & \hspace{-3pt}c_\phi /c_\theta
        \end{bmatrix},
\end{equation*}

\noindent where $\Delta T = 0.1$ and $s_{(\cdot)}$, $c_{(\cdot)}$, and $t_{(\cdot)}$ are short for $\sin(\cdot)$, $\cos(\cdot)$, and $\tan(\cdot)$ respectively. We are given $9 \times 10^6$ training data tuples $(x_i, u_i, f(x_i, u_i))$, where $x_i$, $u_i$ are generated with Halton sampling over $[-1,1]^3 \times [-\frac{\pi}{20},\frac{\pi}{20}]^3$ and $[-1,1]^6$, respectively (data is collected near hover). As $f_0(x)$ is a simple integrator term, we assume it is known and we set $g_0(x) = x$, while $g_1(x)$ is learned with a neural network with one hidden layer of size 4000. We select $a=6$ in Alg. \ref{alg:select_D}. We use $10^6$ more samples in Alg. \ref{alg:lipschitz} to estimate $L_{f-g}$, and conduct a KS test resulting in a $p$-value of $0.43 \gg 0.05$. We obtain $\hat\gamma = 0.205$, $c = 0.011$, and $\epsilon = 0.134$.

See Fig. \ref{fig:quadrotor} for examples of the planned, open-loop, and closed-loop trajectories planned with LMTD-RRT and a na\"ive RRT.
The trajectory planned with LMTD-RRT
remains close to the training data, and the closed-loop system tracks the planned path with $\epsilon$-accuracy converging to $\B_{\epsilon+\lambda}(x_G)$. 
We note that using the feedback controller to track trajectories planned with na\"ive RRT tends to worsen the tracking error, implying our learned model is highly inaccurate outside of the domain. We provide statistics in Table \ref{table:stats_quadrotor} for maximum tracking error and distance to goal, averaged over 100 random start/goal states. The worst case closed-loop tracking error for trajectories planned with LMTD-RRT is $0.011$, again much smaller than $\epsilon$. As with the 2D example, we give the baseline an advantage in computing tracking error statistics by reporting the minimum of the two errors when planning with 1) the same model used in LMTD-RRT and 2) a model trained on the full dataset ($10^7$ points). Despite the data advantage, the plans computed using na\"ive RRT have much higher tracking error. Average planning times for our unoptimized code are 100 sec. for LMTD-RRT and 15 min. for na\"ive RRT, suggesting that sampling focused near the training data can improve planning efficiency.

\begin{figure}
    \centering
    \includegraphics[width=\linewidth]{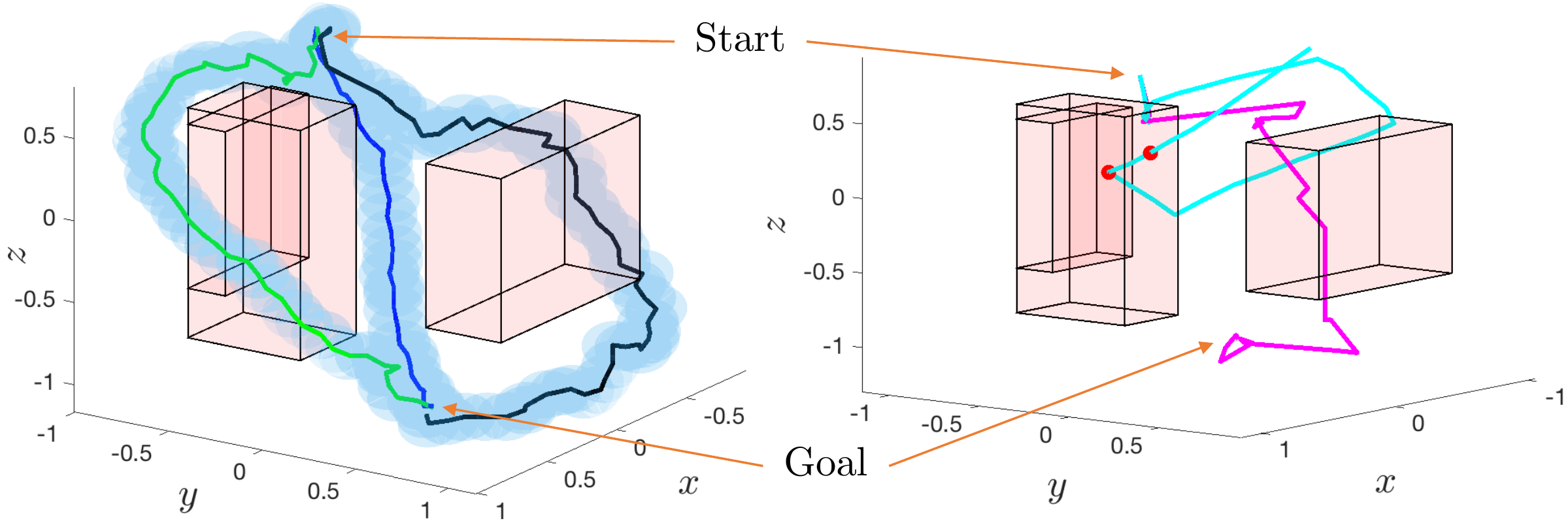}
    \caption{\textbf{Left}: Quadrotor obstacle (red) avoidance. Example plans (green, blue, black), tracking error bound $\epsilon$ overlaid (light blue). Closed-loop trajectories remain in the tubes, converging to the goal without colliding. \textbf{Right}: Na\"ive RRT plan (pink) fails to be tracked (cyan) and collides (red dots).}
    \label{fig:quadrotor_obs}
\end{figure}

We also evaluate LMTD-RRT on an obstacle avoidance problem (Fig. \ref{fig:quadrotor_obs}). We perform collision checking as described in Sec. \ref{sec:safety_stability}. As the tracking error tubes (of radius $\epsilon = 0.134$) centered around the nominal trajectories never intersect with any obstacles, we can guarantee that the system never collides in execution. Empirically, in running Alg. \ref{alg:rrt} over 500 random seeds to obtain different nominal paths, the closed-loop trajectory never collides. In contrast, the na\"ive RRT plan fails to be tracked and collides (Fig. \ref{fig:quadrotor_obs}, right).

\begin{table}\centering
\begin{tabular}{ c | c | c }
  & LMTD-RRT & Na\"ive kino. RRT \\\hline
 \cellcolor{lightgray!50!} \hspace{-7pt}Max. trck. err. (CL)\hspace{-5pt} & \cellcolor{lightgray!50!} 0.003 $\pm$ 0.001 (0.008)& \cellcolor{lightgray!50!} 10.59 $\pm$ 16.75 (153.26) \\ 
 \hspace{-7pt}Goal error (CL)\hspace{-5pt} & 0.001 $\pm$ 0.001 (0.004)& 8.247 $\pm$ 8.434 (46.150) \\ 
 \cellcolor{lightgray!50!} \hspace{-7pt}Max. trck. err. (OL)\hspace{-5pt} & \cellcolor{lightgray!50!} 0.020 $\pm$ 0.007 (0.040)& \cellcolor{lightgray!50!} 4.289 $\pm$ 2.340 (12.986)\\  
 \hspace{-7pt}Goal error (OL)\hspace{-5pt} & 0.019 $\pm$ 0.008 (0.040)& 3.265 $\pm$ 2.028 (11.670)\\  
\end{tabular}
\caption{Quadrotor errors (no obstacles) in closed loop (CL) and \\ open loop (OL). Mean $\pm$ standard deviation (worst case).}
\label{table:stats_quadrotor}
\end{table}

\subsection{7DOF Kuka Arm in Mujoco}

\begin{figure*}[t]
    \centering
    \includegraphics[width=\linewidth]{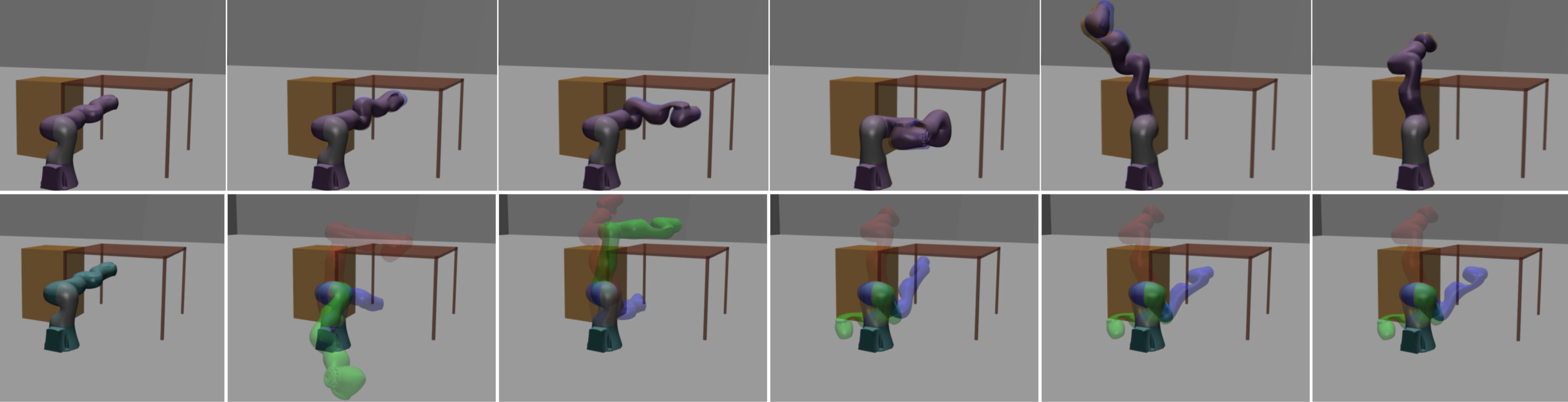}
    \caption{Planning to move a 7DOF arm from below to above a table. Trajectory-tracking time-lapse (time increases from left to right). Red (nominal), green (closed loop), blue (open loop). \textbf{Top}: LMTD-RRT (red, green, blue overlap due to tight tracking). \textbf{Bottom}: Na\"ive RRT (poor tracking causes collision).}
    \label{fig:arm}
\end{figure*}

We evaluate our method on a 7DOF Kuka iiwa arm simulated in Mujoco \cite{todorov2012mujoco} using a Kuka model from \cite{githubkuka}. We train two models using different datasets, one for evaluating tracking error without the presence of obstacles (Table \ref{table:stats_kuka}), and the other for obstacle avoidance. For both models, $g_0(x)$ is again set to be $x$ while $g_1(x)$ is learned with a neural network with one hidden layer of size 4000. For the results in Table \ref{table:stats_kuka}, we are provided 2475 training data tuples, which are collected by recording continuous state-control trajectories from an expert and evaluating $f(x,u)$ on the trajectories and on random state-control perturbations locally around the trajectories. We select $a=5$ in Alg. \ref{alg:select_D}. 275 more samples are used in Alg. \ref{alg:lipschitz} to estimate $L_{f-g}$, validated with a KS test with a $p$ value of $0.58 \gg 0.05$. We obtain $\hat\gamma = 0.087$ and $c = 0.001$, leading to $\epsilon = 0.111$. In Table \ref{table:stats_kuka}, we provide statistics on maximum tracking error and distance to goal under plans with LMTD-RRT and the na\"ive RRT baseline (with a model trained on the full dataset of 2750 points), averaged over 25 runs of each method. Notably, closed-loop tracking of plans found with LMTD-RRT have lowest error, with a worst case error much smaller than $\epsilon = 0.111$. Planning takes on average 1.552 and 0.167 sec. for LMTD-RRT and na\"ive RRT, respectively. We suspect the na\"ive RRT exploits poor dynamics outside of $D$, expediting planning.

\begin{table}\centering
\begin{tabular}{ c | c | c }
  & LMTD-RRT & Na\"ive kino. RRT \\\hline
 \cellcolor{lightgray!50!} \hspace{-7pt}Max. trck. err. (CL)\hspace{-5pt} & \cellcolor{lightgray!50!} 0.010 $\pm$ 0.010 (0.038)& \cellcolor{lightgray!50!} 0.090 $\pm$ 0.250 (1.265) \\ 
 \hspace{-7pt}Goal error (CL)\hspace{-5pt} & 0.004 $\pm$ 0.004 (0.018) & 0.082 $\pm$ 0.251 (1.265) \\ 
 \cellcolor{lightgray!50!} \hspace{-7pt}Max. trck. err. (OL)\hspace{-5pt} & \cellcolor{lightgray!50!} 0.036 $\pm$ 0.041 (0.142) & \cellcolor{lightgray!50!} 0.076 $\pm$ 0.084 (0.282)\\ 
 \hspace{-7pt}Goal error (OL)\hspace{-5pt} & 0.035 $\pm$ 0.040 (0.140) & 0.071 $\pm$ 0.075 (0.225)\\  
\end{tabular}
\caption{7DOF arm errors (no obstacles) in closed loop (CL) and \\ open loop (OL). Mean $\pm$ standard deviation (worst case).}
\label{table:stats_kuka}
\end{table}

For the obstacle avoidance example (Fig. \ref{fig:arm}), we are provided 15266 datapoints, which again take the form of continuous trajectories plus perturbations. We select $a=8$ in Alg. \ref{alg:select_D}, and use 1696 more points to estimate $L_{f-g}$ using Alg. \ref{alg:lipschitz}, which we validate with a KS test with a $p$ value of $0.37 \gg 0.05$. We obtain $\hat\gamma = 0.156$ and $c = 0.010$, leading to $\epsilon = 0.111$. In planning, as described in Sec. \ref{sec:safety_stability}, we perform collision checking by randomly sampling configurations in an $\epsilon$-ball about each point along the trajectory. Though this collision checker is not guaranteed to detect collision, in running LMTD-RRT over 20 random seeds, we did not observe collisions in execution for any of the 20 plans, and the arm safely reaches the goal without collision. Over these trajectories, the worst case tracking error is $0.107$, which remains within $\epsilon = 0.111$. One such plan computed by LMTD-RRT and the corresponding open-loop and closed-loop tracking trajectories, is shown in the top row of Fig. \ref{fig:arm}. The three trajectories nearly overlap exactly due to small tracking error. In contrast, the na\"ive RRT plan cannot be accurately tracked, even with closed-loop control, due to planning outside of the trusted domain, causing the executed trajectories to diverge and collide with the table.

\section{DISCUSSION AND CONCLUSION}

We present a method to bound the difference between learned and true dynamics in a given domain and derive conditions that guarantee a one-step feedback law exists. We combine these two properties to design a planner that can guarantee safety, goal reachability, and that the closed-loop system remains in a small region about the goal.

While the method presented has strong guarantees, it also has limitations which are interesting targets for future work. First, the true dynamics are assumed to be deterministic. Stochastic dynamics may be possible by estimating the Lipschitz constant of the mean dynamics while also appropriately modeling the noise. Second, 
the actuation requirement limits the systems that this method can be applied to. 
For systems with $\texttt{dim}(\U) < \texttt{dim}(\X)$, it may be possible to construct a similar feedback law that guarantees the learned dynamics will lie within a tolerance of planned states which, in turn, could still give strong guarantees on safety and reachability.

\bibliographystyle{IEEEtran}
\bibliography{references.bib}

\end{document}